\documentclass{article} %
\usepackage[final]{colm2025_conference}

\usepackage{microtype}
\usepackage{hyperref}
\usepackage{url}
\usepackage{booktabs}

\usepackage{lineno}

\definecolor{darkblue}{rgb}{0, 0, 0.5}
\hypersetup{colorlinks=true, citecolor=darkblue, linkcolor=darkblue, urlcolor=darkblue}

\title{Local Mixtures of Experts: \\ Essentially Free Test-Time Training via Model Merging}

\author{Ryo Bertolissi\thanks{Equal contribution. Correspondence to \texttt{jonas.huebotter@inf.ethz.ch}.}, \,Jonas Hübotter\footnotemark[1], \,Ido Hakimi, Andreas Krause \\
ETH Z\"urich, Switzerland
}

\usepackage{amsmath}
\usepackage{amssymb}
\usepackage{mathtools}
\usepackage{amsthm}
\usepackage{bbm}
\usepackage{xspace}
\usepackage{nicefrac}
\usepackage{algorithm}
\usepackage{algpseudocode}  %
\usepackage{wrapfig}
\usepackage{subcaption}
\usepackage{titletoc}
\usepackage{multirow}
\usepackage{xcolor,colortbl}
\usepackage{arydshln}
\usepackage[framemethod=tikz]{mdframed}
\usepackage[frozencache]{minted}
\definecolor{lightgray}{gray}{0.95}

\usepackage[capitalize,noabbrev]{cleveref}

\newcounter{insight}
\DeclareRobustCommand{\insight}{%
  \refstepcounter{insight}%
  Insight~\theinsight:\xspace%
}

\usepackage{todonotes}

\usepackage{titlesec}
\titlespacing*{\paragraph}{0pt}{0.25ex}{2ex}

\definecolor{chaptercolor}{HTML}{1A254B}
\definecolor{darkblue}{HTML}{1A254B}
\definecolor{linkcolor}{HTML}{2B50AA}
\definecolor{citecolor}{HTML}{2B50AA}
\definecolor{lightlinkcolor}{HTML}{9A8F97}
\definecolor{darklinkcolor}{HTML}{1A254B}
\definecolor{light}{HTML}{F8F8F8}
\definecolor{lightblue}{HTML}{A7BED3}
\definecolor{red}{HTML}{F2545B}
\definecolor{blue}{HTML}{2b50aa}

\theoremstyle{plain}
\newtheorem{theorem}{Theorem}[section]
\newtheorem{proposition}[theorem]{Proposition}
\newtheorem{informalproposition}[theorem]{Informal Proposition}

\newcommand{\figref}[2]{Figure~\hyperref[#1]{\ref{#1} (#2)}}

\usepackage{import}
\usepackage{xifthen}
\usepackage{pdfpages}
\usepackage{transparent}

\NewDocumentCommand{\incfig}{mo}{
  \begin{center}
    \IfValueT{#2}{\def\svgwidth{#2}}{\def\svgwidth{\columnwidth}}
    \import{./figures/}{#1.pdf_tex}
  \end{center}
}

\usepackage{pgf}
\usepackage{adjustbox}

\NewDocumentCommand{\incplt}{O{\columnwidth}m}{%
  \begin{center}
    \adjustbox{center}{\adjustbox{width=#1+10pt}{\includegraphics[width=#1]{./plots/output/#2.pdf}}}
  \end{center}
}

\newcommand*{\abs}[1]{| #1 |}

\NewDocumentCommand{\norm}{sm}{\IfBooleanTF{#1}{\|#2\|}{\left\| #2 \right\|}}
\NewDocumentCommand{\normF}{sm}{\IfBooleanTF{#1}{\|#2\|_{\mathrm{F}}}{\left\| #2 \right\|_{\mathrm{F}}}}
\NewDocumentCommand{\dTV}{sm}{d_{\mathrm{TV}}\IfBooleanTF{#1}{(#2)}{\left( #2 \right)}}

\DeclareMathOperator*{\defeq}{\,\dot{=}\,}

\DeclareMathOperator*{\diam}{diam}

\DeclareMathOperator{\relu}{ReLU}
\DeclareMathOperator{\softmax}{softmax}
\DeclareMathOperator{\ssoftmax}{sparse-softmax}

\DeclarePairedDelimiter\parentheses{(}{)}
\DeclarePairedDelimiter\brackets{[}{]}
\DeclarePairedDelimiter\braces{\{}{\}}

\newcommand{\R}{\mathbb{R}}

\renewcommand{\vec}[1]{\boldsymbol{#1}}
\newcommand{\mat}[1]{\boldsymbol{#1}}

\newcommand{\spa}[1]{\mathcal{#1}}
\newcommand{\opt}[1]{#1^\star}

\NewDocumentCommand{\irred}{som}{\ensuremath{\sigma_{\hspace{-1pt}\infty}\IfBooleanTF{#1}{^2}{}(#3\IfValueTF{#2}{;#2}{})}}

\NewDocumentCommand{\fnPr}{}{\mathbb{P}}
\RenewDocumentCommand{\Pr}{om}{\fnPr\IfValueT{#1}{_{#1}}\parentheses*{#2}}
\RenewDocumentCommand{\H}{mo}{\mathrm{H}\IfValueTF{#2}{\!\left[#1\ \middle|\ #2\right]}{\brackets*{#1}}}
\NewDocumentCommand{\Hsm}{mo}{\mathrm{H}\IfValueTF{#2}{[#1 \mid #2]}{\brackets{#1}}}
\NewDocumentCommand{\I}{mmo}{\mathrm{I}\IfValueTF{#3}{\!\left(#1;#2\ \middle|\ #3\right)}{\parentheses*{#1; #2}}}
\NewDocumentCommand{\Ism}{mmo}{\mathrm{I}\IfValueTF{#3}{(#1;#2 \mid #3)}{\parentheses{#1; #2}}}

\NewDocumentCommand{\E}{somo}{\ensuremath{\mathbb{E}\IfValueT{#2}{_{#2}}{} \IfBooleanTF{#1}{#3}{\IfValueTF{#4}{\!\left[#3\ \middle|\ #4\right]}{\brackets*{#3}}}}}
\NewDocumentCommand{\Esm}{somo}{\ensuremath{\mathbb{E}\IfValueT{#2}{_{#2}}{} \IfBooleanTF{#1}{#3}{\IfValueTF{#4}{\!\left[#3\ \middle|\ #4\right]}{\brackets{#3}}}}}
\NewDocumentCommand{\Var}{somo}{\mathrm{Var}\IfValueT{#2}{_{#2}}{} \IfBooleanTF{#1}{#3}{\IfValueTF{#4}{\!\left(#3\ \middle|\ #4\right)}{\parentheses*{#3}}}}
\NewDocumentCommand{\Varsm}{somo}{\mathrm{Var}\IfValueT{#2}{_{#2}}{} \IfBooleanTF{#1}{#3}{\IfValueTF{#4}{\left(#3\ \middle|\ #4\right)}{\parentheses{#3}}}}
\NewDocumentCommand{\Cov}{som}{\mathrm{Cov}\IfValueT{#2}{_{#2}}{} \IfBooleanTF{#1}{#3}{\brackets*{#3}}}
\NewDocumentCommand{\Cor}{som}{\mathrm{Cor}\IfValueT{#2}{_{#2}}{} \IfBooleanTF{#1}{#3}{\brackets*{#3}}}

\NewDocumentCommand{\grad}{e_}{\boldsymbol{\nabla}\IfValueT{#1}{_{\!\!#1}\,}}

\NewDocumentCommand{\diag}{som}{\mathrm{diag}\IfValueT{#2}{_{#2}}{} \IfBooleanTF{#1}{\braces{#3}}{\braces*{#3}}}

\NewDocumentCommand{\N}{somm}{\mathcal{N}\IfBooleanTF{#1}{\left(}{(}\IfValueT{#2}{#2;}{} #3, #4\IfBooleanTF{#1}{\right)}{)}}
\NewDocumentCommand{\GP}{omm}{\mathcal{GP}(\IfValueT{#1}{#1;}{} #2, #3)}

\newcommand{\vx}{\vec{x}}

\newcommand{\vxs}{\vec{\opt{x}}}

\newcommand{\vz}{\vec{z}}

\newcommand{\vphi}{\boldsymbol{\phi}}

\newcommand{\vtheta}{\boldsymbol{\theta}}

\newcommand{\mA}{\mat{A}}
\newcommand{\mB}{\mat{B}}

\newcommand{\mPhi}{\mat{\Phi}}

\newcommand{\mTheta}{\mat{\Theta}}

\newcommand{\mW}{\mat{W}}

\newcommand{\spD}{\spa{D}}

\newcommand{\spL}{\spa{L}}

\newcommand{\spX}{\spa{X}}
\newcommand{\spY}{\spa{Y}}

\begin{document}

\ifcolmsubmission
\linenumbers
\fi

\maketitle

\begin{abstract}
    Mixture of expert~(MoE) models are a promising approach to increasing model capacity without increasing inference cost, and are core components of many state-of-the-art language models.
    However, current MoE models typically use only few experts due to prohibitive training and inference cost.
    We propose \emph{\textbf{T}est-\textbf{T}ime \textbf{M}odel \textbf{M}erging}~(TTMM) which scales the MoE paradigm to an order of magnitude more experts and uses model merging to avoid almost any test-time overhead.
    We show that TTMM is an approximation of test-time training~(TTT), which fine-tunes an expert model for each prediction task, i.e., prompt.
    TTT has recently been shown to significantly improve language models, but is computationally expensive.
    We find that performance of TTMM improves with more experts and approaches the performance of TTT.
    Moreover, we find that with a 1B parameter base model, \emph{TTMM is more than $100\times$ faster than TTT} at test-time by amortizing the cost of TTT at train-time.
    Thus, TTMM offers a promising cost-effective approach to scale test-time training.\looseness=-1%
\end{abstract}

\begin{figure}[h]
  \vspace{-1ex}
  \centering
  \incplt[0.85\textwidth]{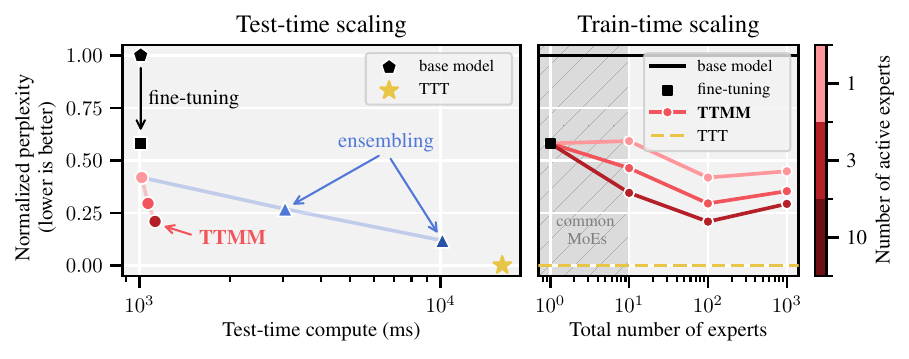}
  \vspace{-2.5ex}
  \caption[]{Accuracy gains of TTMM when scaling test-time compute~(left) and pre-training compute~(right).
  The normalized perplexity between base model and TTT model is averaged across evaluation datasets and models.
  \textbf{Left:}~We measure language modeling ability against the time to generate 100 tokens. TTMM approaches the performance of TTT without almost any test-time overhead when increasing the number of active experts. Ensembling rather than merging experts improves performance, but at a much higher test-time cost. Standard fine-tuning on the training data improves performance without any test-time overhead, but does not yield close to the performance of TTT.
  \textbf{Right:}~We measure the performance of TTMM against the total number of experts. Whereas common MoE models use few experts, TTMM scales to an order of magnitude more experts, improving performance with more experts. Scaling the number of experts in TTMM does not affect training FLOPs or required GPU memory.\footnotemark\ Training the individual experts is embarrassingly parallelizable.\looseness=-1}
  \label{fig:results}
  \vspace{-1ex}
\end{figure}
\footnotetext{Scaling the number of experts increases CPU host memory, which is cheaper than GPU memory.}

\section{Introduction}

The standard paradigm of machine learning separates training and testing.
In training, we learn a model by \emph{inductively} extracting general rules from data, while in testing, we apply this model to new unseen data.
We study an alternative \emph{transductive} paradigm called \emph{test-time training}~\citep[TTT,][]{sun2020test} where we use a specific expert model for each prediction task.
Variations of this paradigm have been studied since the inception of the machine learning field~\citep{cleveland1979robust,cleveland1988locally,atkeson1997locally,bottou1992local}.
More recently, fine-tuning large pre-trained neural networks at test-time has regained interest in computer vision~\citep[e.g.,][]{jain2011online,shocher2018zero,sun2020test,gandelsman2022test,ruiz2023dreambooth} and language modeling~\citep{krause2018dynamic,krause2019dynamic,hardt2023test,sun2024learning,huebotter2025efficiently}.
\cite{hardt2023test} and \cite{huebotter2025efficiently} fine-tune a pre-trained language model on data related to the prompt, and show that this can significantly improve language modeling performance.
Beyond pure language modeling, TTT has been instrumental in state-of-the-art approaches to abstract reasoning~\citep{akyurek2024surprising}, video generation~\citep{dalal2025one}, and also been shown to improve reasoning language models~\citep{zuo2025ttrl,simonds2025ladder}.\looseness=-1

While TTT can significantly improve model performance, it comes with a high computational cost at test-time: TTT requires fine-tuning the model for every task, i.e., each prompt.
We propose \emph{\textbf{T}est-\textbf{T}ime \textbf{M}odel \textbf{M}erging}~(TTMM) which amortizes the cost of TTT at train time: \begin{itemize}\vspace{-3pt}
    \item At \textbf{train-time}, TTMM clusters the training data into many local neighborhoods and trains separate small expert LoRA adapters~\citep{hu2021lora} for each cluster.
    \item At \textbf{test-time}, TTMM dynamically selects a subset of LoRA adapters and merges their parameters to form a single task-specific model.
\end{itemize}\vspace{-3pt}
In our experiments, TTMM approaches the language modeling performance of TTT without incurring almost any significant compute or memory cost~(cf.~\cref{fig:results}).
\Cref{fig:ttt_vs_ttmm} illustrates the difference between TTT and TTMM.\looseness=-1

\begin{figure}
    \centering
    \includegraphics[width=0.8\textwidth]{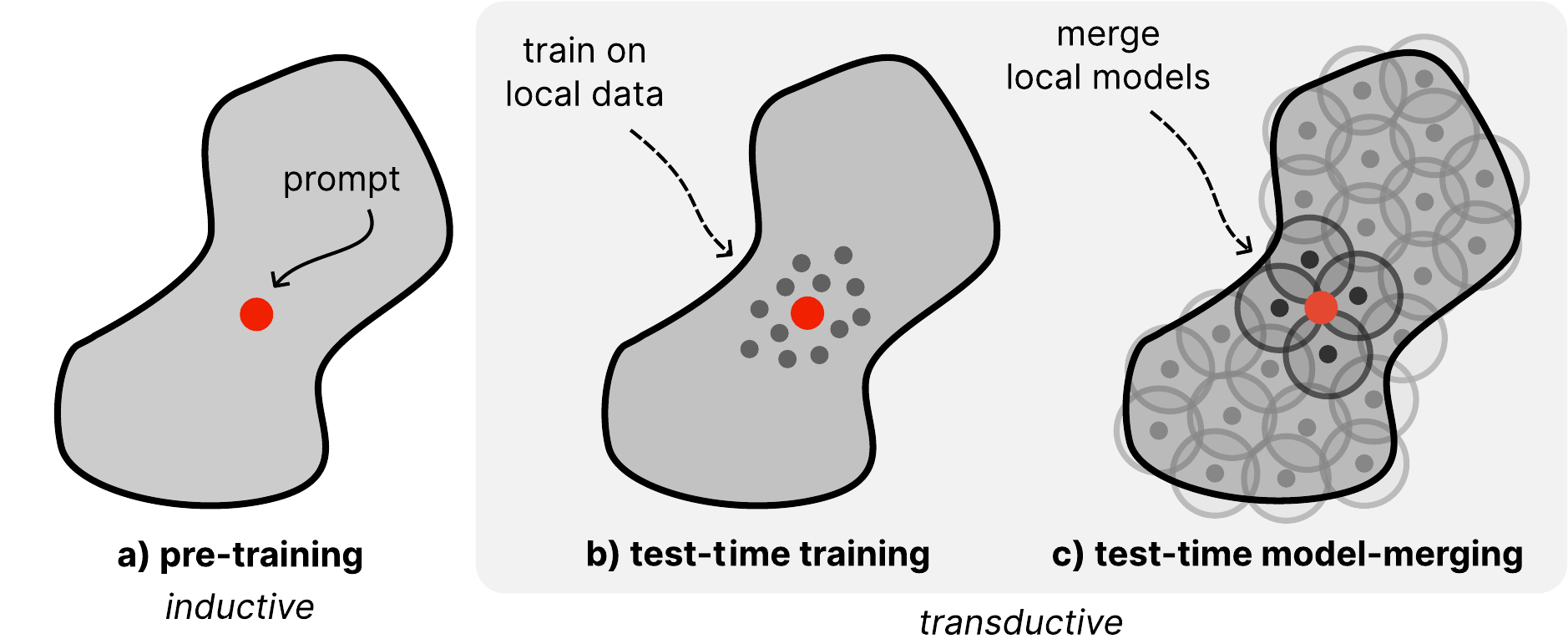}
    \caption{Standard \emph{inductive} language modeling pre-trains a single model on the dataset~(shown in \textcolor{gray}{gray}) which is then used for prediction. Test-Time Training~(TTT) improves performance by fine-tuning the model at test-time on data related to the prompt, but is computationally expensive. Test-Time Model Merging~(TTMM) fine-tunes the pre-trained model to many local neighborhoods at train-time, and merges the local models related to the prompt at test-time. Both TTT and TTMM target models specifically to the prompt~(i.e., are \emph{transductive}), yet, TTMM does not require expensive fine-tuning at test-time.\looseness=-1}
    \label{fig:ttt_vs_ttmm}
\end{figure}

Another extensive body of work studies merging multiple models to combine distinct capabilities and knowledge~\citep{yang2024model}.
We discuss the connection of TTMM to this literature on model merging and mixture of experts more extensively in \cref{sec:related_work}.
Model merging is typically used in a multitask setting with \emph{few} ``meta-tasks'' (e.g., coding, math, law, etc.) and corresponding expert models.
This multitask setting differs from the TTT setting where each individual prompt is considered its own task.
Hence, TTMM trains \emph{many} local models, even within a single meta-task (say for all kinds of different coding tasks).
Then, for \emph{each} prompt at test-time, TTMM merges the most related local models.
As such, TTMM can be seen as bridging multitask model merging and TTT by dynamically constructing an expert model at test-time from an order of magnitude more meta-tasks than in multitask learning.\looseness=-1

We contribute the following key findings, summarized in \cref{fig:results}: \begin{enumerate}\vspace{-3pt}
    \item \textbf{TTMM outperforms multitask model merging:} We find that training and merging \emph{many} expert models, even within a single domain such as ``coding in Python'', outperforms training and merging \emph{few} models specialized on broader tasks.\looseness=-1
    \item \textbf{TTMM approaches performance of TTT without almost any test-time overhead:} We discuss how to merge the parameters of the expert models at almost no additional compute or memory cost.
    We then evaluate TTMM on the Wikipedia and GitHub Python corpora and find that it achieves performance close to TTT, while matching the performance of ensembling without incurring its test-time overhead.\looseness=-1
\end{enumerate}

\section{Related Work}\label{sec:related_work}

We propose TTMM as an efficient amortization of TTT, which realizes comparable accuracy.
We next discuss how TTMM can be seen as an extension to the literature on model selection and model merging.
In \cref{sec:additional_related_work}, we discuss the link to retrieval-augmented generation.\looseness=-1

\paragraph{Multitask Model Merging.}
Combining the predictions of multiple models has been a long-standing practice in machine learning, e.g., in the form of majority voting~\citep{wang2022self} or by ensembling of predictions~\citep{dietterich2000ensemble,wortsman2022model}, which has been used to avoid catastrophic forgetting~\citep[e.g.,][]{bagatella2025active}.
However, merging predictions of many neural networks is computationally inefficient since it requires a separate forward pass for each model.
Instead, TTMM merges in parameter-space, and only requires a single forward pass of the merged model, resulting in a negligible increase of compute and memory cost compared to inferencing a single model.
The idea of merging multiple models by averaging their parameters (not their predictions!) is grounded in the surprising phenomenon of \emph{mode-connectivity}~\citep{garipov2018loss,wilson2025deep}, which observes that modes in ensembles of models are connected by paths of small loss.
This has motivated a variety of model merging techniques, including stochastic weight averaging~\citep{izmailov2018averaging} and model soups~\citep{wortsman2022model}.
Recently, model merging has been studied predominantly in the context of merging \emph{few} expert models with the aim of retaining performance on their respective meta-tasks~\citep{ilharco2022editing,matena2022merging,jin2022dataless,ainsworth2022git}.
Merging model parameters may lead to interference, which has been a key focus of recent work~\citep{yadav2023ties,yang2023adamerging,tang2023parameter,daheim2023model,yu2024language,jung2024tint}.\looseness=-1

\paragraph{Dynamic Merging of Multitask Models.}
Recent work in multitask model merging studies selecting merging coefficients dynamically depending on the prompt.
\cite{oh2024dawin} select coefficients based on the entropy of the individual model's predictions, which is computationally expensive since it requires forward passes of all networks.
Instead, TTMM computes merging coefficients based on a measure of similarity between prompt and experts which is tuned on holdout data at train-time~\citep{lu2025twin,tang2024merging,cheng2024dam,lai2025mediator}, and which incurs only negligible test-time overhead.
In the tangentially related literature on black-box model fusion, dynamic model selection has been widely studied~\citep[e.g.,][]{liu2021dexperts,jiang2023llm,mavromatis2024pack}.
However, in the setting of TTMM, we control data and architecture used for training the local models, which enables us to compute merging coefficients more efficiently than in black-box model fusion.\looseness=-1

\paragraph{Clustering and Dynamic Model Selection.}
While the above literature focuses on multitask learning with typically only a handful of expert models, TTMM clusters the training data into \emph{many} local neighborhoods and trains a separate local model for each cluster.
Like TTMM, Branch-Train-Merge~\citep[BTM,][]{li2022branch,gururangan2023scaling} clusters the training data into local neighborhoods, trains a separate local model for each cluster, and dynamically selects a subset of these models at test-time.
However, unlike BTM, TTMM uses an order of magnitude more local models and merges their parameters instead of their predictions, leading to substantially faster inference.
In concurrent work, \cite{qiu2025mingle} study TTMM in a continual learning setting.\looseness=-1

\paragraph{Mixture of Experts.}
The mixture of experts \citep[MoE,][]{shazeer2017outrageously} paradigm increases model capacity without increasing inference cost by introducing conditional routing of inputs to subsets of learnable parameters (i.e., ``experts'').
Training MoE models is challenging due to the need for joint training of the experts and the router on substantial multitask data.
Moreover, MoE models evaluate all active experts individually, which can still require substantial memory.
This in turn implies that the number of experts is typically limited to a small number.
Recent studies have focused on challenges such as reducing synchronization~\citep{sukhbaatar2024branch,zhang2024bam,zhang2025bts}, load balancing of experts~\citep{clark2022unified,zhou2022mixture}, and improving expert specialization~\citep{dai2024deepseekmoe}.
In contrast, TTMM approximates the MoE paradigm by embarrassingly parallelizing the training of the local models and merging model parameters at test-time, which permits scaling to many orders of magnitude more experts and which enables approximation of the ``maximally local'' TTT.\looseness=-1

\section{TTMM: Test-Time Model Merging}

Our proposed method, TTMM, has two main stages: train-time and test-time.
At train-time, TTMM clusters the training data into many local neighborhoods and trains a separate expert model for each cluster.
Then at test-time, TTMM dynamically composes these expert models to form a single task-specific model.
We discuss both stages in more detail below.\looseness=-1

\begin{figure}
    \centering
    \includegraphics[width=0.8\textwidth]{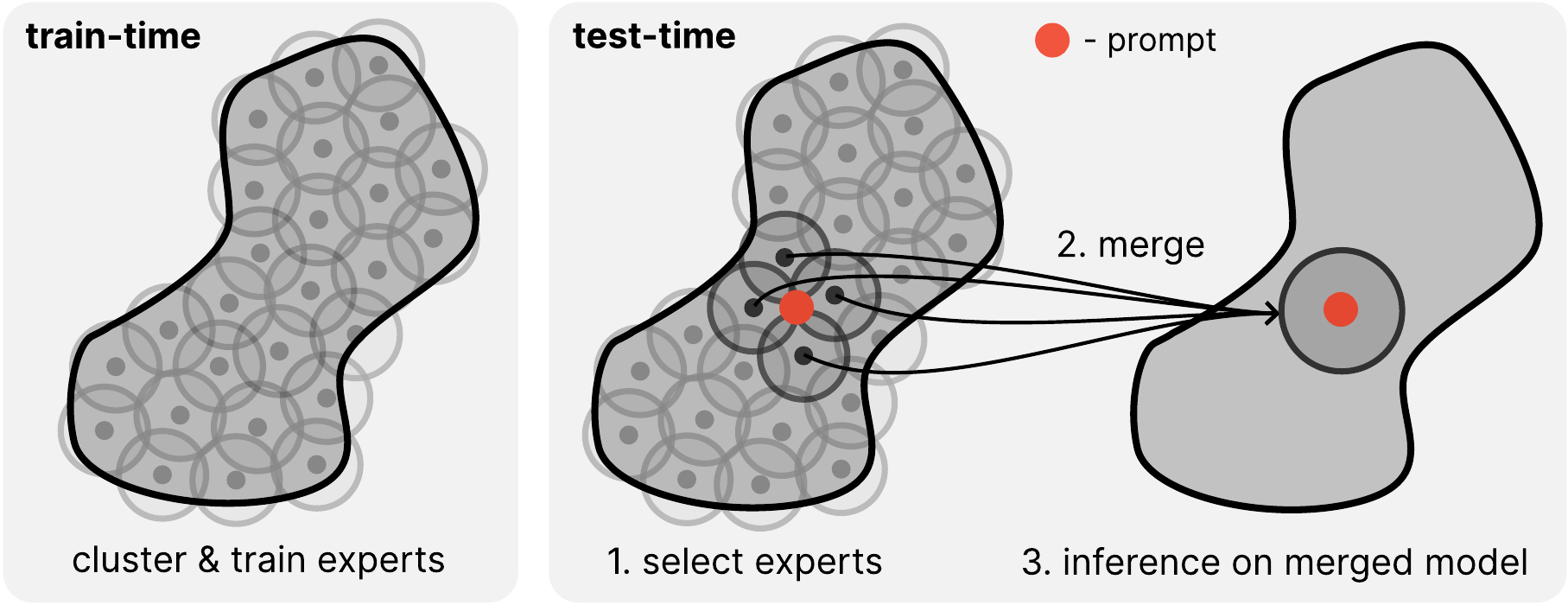}
    \caption{Illustration of \textbf{TTMM}: At \emph{train-time} (\cref{alg:ttmm_train_time}), TTMM clusters the training data (shown in \textcolor{gray}{gray}) into many local neighborhoods and trains a separate expert model for each cluster. At \emph{test-time} (\cref{alg:ttmm_test_time}), TTMM dynamically selects a subset of expert models related to the prompt and merges their parameters to form a single task-specific model. The test-time stage is compute and memory efficient, using the previously trained experts.\looseness=-1}
    \label{fig:ttmm}
\end{figure}

We consider a domain~$\spX$ of token sequences and a training set $\spD \subseteq \spX$.
We further assume that we have access to a pre-trained autoregressive language model~$f(\vx; \vtheta)$ that maps token sequences $\vx \in \spX$ to probability distributions over the next token.
Finally, we assume access to a sequence embedding model~$\vphi(\vx)$ returning normalized embeddings.
In our experiments, we use a mean-pooled sequence embedding model, but one could also use the last-layer token embeddings from the pre-trained language model~\citep[see][]{lu2025twin}.\looseness=-1

\begin{wrapfigure}{r}{0.525\textwidth}
\vspace{-5ex}
\begin{minipage}{0.525\textwidth}
    \begin{algorithm}[H]
    \caption{TTMM at train-time}
    \label{alg:ttmm_train_time}
    \small
    \begin{algorithmic}[1]
        \Require{
        Language model $f(\vx; \vtheta)$ with pre-trained weights $\vtheta$; sequence embedding model $\vphi(\vx)$; training set $\spD$; number of clusters $K$.
        }
        \vspace{\baselineskip}

        \State Cluster training set $\spD$ into $K$ clusters $\spD_1, \dots, \spD_K$.
        \State Train a LoRA adapter $\vtheta_k$ on each cluster $\spD_k$.
        \State Compute cluster-specific embeddings:
        \vspace{-1.5ex}\[
            \textstyle \vphi_k \gets \mathrm{normalize}\parentheses*{\frac{1}{|\spD_k|} \sum_{\vx \in \spD_k} \vphi(\vx)}
        \]\vspace{-2ex}
    \end{algorithmic}
    \end{algorithm}
\end{minipage}
\vspace{-3ex}
\end{wrapfigure}

\paragraph{Train-time:~Clustering.}
To cluster the training dataset, we use bisecting $k$-means, a hierarchical clustering algorithm which recursively applies $k$-means with $k=2$ to the cluster with the largest diameter~\citep{jain2010data}.
We use bisecting $k$-means since it is more efficient than $k$-means and tends to lead to more balanced clusters, however, any clustering algorithm can be used.
We then train a small LoRA adapter on each cluster for one epoch.
Finally, we average the embeddings of the sequences in each cluster to obtain cluster-specific embeddings (i.e., ``centroids'') that represent the capabilities of the corresponding expert model.
We renormalize the centroids to avoid biasing the embedding norm of larger clusters toward zero.
Without the additional normalization, centroids of larger clusters have smaller norm despite~$\vphi(\vx)$ being normalized.\looseness=-1

\begin{wrapfigure}{r}{0.525\textwidth}
\vspace{-4.5ex}
\begin{minipage}{0.525\textwidth}
    \begin{algorithm}[H]
    \caption{TTMM at test-time}
    \label{alg:ttmm_test_time}
    \small
    \begin{algorithmic}[1]
        \Require{
            Prompt $\vxs$ with embedding $\vphi^\star$. Language model $f(\vx; \vtheta)$ with pre-trained weights $\vtheta$; expert models $\smash{\{(\vtheta_k, \vphi_k)\}_{k=1}^K}$; temperature~$\beta$ and sparsity~$\tau$ (tuned on holdout data).
        }
        \vspace{\baselineskip}

        \State Compute cluster-specific merging coefficients:
        \vspace{-1.5ex}\[
            \textstyle w_k \gets \ssoftmax_\tau\parentheses*{\frac{1}{\beta} \vphi_k^\top \vphi^\star}
        \]\vspace{-2.5ex}
        \State Merge into a single expert model:
        \vspace{-1.5ex}\[
            \textstyle \vtheta^\star \gets \sum_{k=1}^K w_k \vtheta_k
        \]\vspace{-3ex}
        \State Use $f(\vxs; \vtheta^\star)$ to generate the next token.
    \end{algorithmic}
    \end{algorithm}
\end{minipage}
\vspace{-4ex}
\end{wrapfigure}

\paragraph{Test-time: Model Merging.}
At test-time, given a prompt $\vxs \in \spX$ and its embedding $\vphi^\star \defeq \vphi(\vxs)$, we compute merging coefficients $w_k$ for each expert model $\vtheta_k$ using a sparse cross-attention mechanism with sparsity parameter $\smash{\tau \in [0,\frac{1}{K})}$: \begin{align*}
  \begin{multlined}[t]
      \ssoftmax_\tau(\vz) \defeq \\ \frac{\relu(\softmax(\vz) - \tau)}{\sum_k \relu(\softmax(\vz)_k - \tau)}
  \end{multlined}
\end{align*} where $\relu(\vz) = \max\{0, \vz\}$.
This sparse softmax returns a sparse distribution over the experts, where experts with weight below $\tau$ are pruned, and the remaining weights are renormalized.
Pruning experts with low weights improves test-time efficiency since less expert parameters need to be merged~(see~\figref{fig:efficiency}{right}), with minimal degradation of accuracy.
The sparsity being smaller than $\nicefrac{1}{K}$ ensures that always at least one expert is selected.
The model selection can easily be parallelized for all clusters and multiple queries by writing the computation as a cross-attention matrix product: \begin{align*}
    \mTheta^\star \gets \ssoftmax_\tau\parentheses*{\frac{1}{\beta} \mPhi^\star \mPhi^\top} \mTheta
\end{align*} with keys $\mPhi \defeq [\vphi_1, \dots, \vphi_K]^\top \in \R^{K \times d}$, queries $\mPhi^\star \defeq [\vphi_1^\star, \dots, \vphi_n^\star]^\top \in \R^{n \times d}$, and values $\mTheta \defeq [\vtheta_1, \dots, \vtheta_K]^\top \in \R^{K \times p}$, computing the collection of $n$ merged expert models $\mTheta^\star \in \R^{n \times p}$.\looseness=-1

\subsection{Latency}

In contrast to MoE models, the number of experts in TTMM is typically too large for all experts to fit in GPU memory at once.
For example, with Llama-3.2-1B as base model, each LoRA adapter has size approximately~$172$~MB, and with $1000$ clusters, the expert models would require prohibitive $172$~GB of GPU memory.
Instead, we dynamically load active experts into GPU memory and merge their parameters before inference.\looseness=-1

\begin{figure}
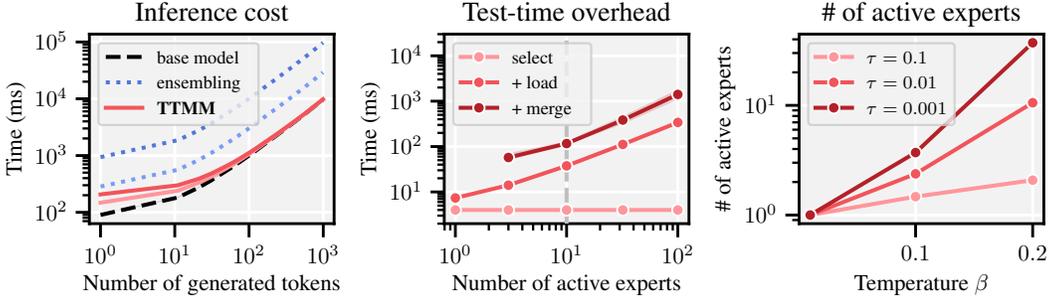

  \vspace{-2.5ex}
    \centering
    \incplt{efficiency}
    \vspace{-2.5ex}
    \caption{
        TTMM has negligible compute overhead at test-time. \textbf{Left:}~Inference cost of ensembling and TTMM with $10$ (dark) and $3$ (light) active experts, respectively. TTMM with $10$ experts costs around $20$ generated tokens, while with $3$ experts, TTMM costs only about $10$ generated tokens. In contrast, ensembling has multiplicative overhead that grows with the number of active experts. \textbf{Middle:}~Test-time overhead of TTMM compared to inference with the base model. When using up to $10$ experts, the test-time overhead is small. \textbf{Right:}~Number of active experts with a given temperature~$\beta$ and sparsity~$\tau$ on the Wikipedia dataset with $K=1000$ experts. \textbf{Setting:}~We use Llama-3.2-1B as base model and consider LoRA adapters with rank $64$. Each LoRA adapter has size $\approx\!172$~MB. For loading, we assume a CPU to GPU bandwidth of $50$~GB/s. For computing embeddings, we use the \texttt{all-mpnet-base-v2} model. We measure selection and merging on an RTX 4090 GPU with $10$ seeds.\looseness=-1}
    \label{fig:efficiency}
\end{figure}

The test-time overhead of TTMM consists therefore of the time required to compute the merging coefficients~(\emph{select}), the time to load active experts from CPU to GPU memory~(\emph{load}), and the time to merge the expert parameters~(\emph{merge}).
In \figref{fig:efficiency}{middle}, we benchmark this overhead across varying numbers of active experts.
On our infrastructure and with $10$~active experts, selection takes roughly $5$~milliseconds, loading takes around $30$~milliseconds, and merging takes approximately $80$~milliseconds.
We use PyTorch's~\texttt{einsum}~\citep{paszke2019pytorch} for computing the merged LoRA adapter $\Delta\mW = \sum_k w_k \mB_k \mA_k$~(cf.~\cref{fig:merging_code} in the appendix), with $\mA_k$ and $\mB_k$ the low-rank factors of the LoRA adapter $\vtheta_k$.
This is around $20\%$ faster than the naive implementation with a for-loop over the experts, without any additional memory overhead.
We anticipate that test-time overhead may be reduced further by interleaving loading and merging by overlapping computation and communication~\citep[see, e.g.,][]{tang2025merging}.\looseness=-1

\paragraph{TTMM costs $\approx$20 tokens.}
In \figref{fig:efficiency}{middle}, we evaluate the practical test-time overhead of TTMM for language generation.
TTMM incurs only a small fixed overhead which, when merging $10$ experts, is equal to generating around $20$ tokens.
When generating more than $100$ tokens, the overhead of TTMM compared to inference with the base model is negligible.
In contrast, ensembling has a multiplicative overhead that grows with the number of active experts.
Comparing to TTT, a single test-time gradient step with Llama-3.2-1B takes approximately $150$~milliseconds~\citep{huebotter2025efficiently}.
With $100$ test-time gradient steps, the total test-time overhead (without any data selection and loading) is already around $15$~seconds.
In contrast, \emph{TTMM with $10$ active experts has a constant overhead of around $115$~milliseconds, which is a more than $125\times$~speedup compared to TTT}.\looseness=-1

\paragraph{Balancing Latency and Performance.}
The exact Pareto frontier of latency and performance varies depending on generation length. \Cref{fig:results} plots the Pareto frontier for 100 generated tokens and \figref{fig:efficiency}{left} shows how the latency changes for other generation lengths. We do not show TTT in \figref{fig:efficiency}{left} since the constant overhead of TTT before generating any token is larger than the time to generate 1k tokens with TTMM (and 10 active experts).\looseness=-1

\subsection{Test-Time Model Merging approximates Test-Time Training}\label{sec:approximation}

TTMM can be seen as an approximation of TTT that amortizes TTT's compute cost at train-time.
We formalize this connection by showing that in the limit of exponentially many experts, TTMM closely approximates TTT~\citep{hardt2023test}, which fine-tunes the pre-trained model on the $N$ nearest neighbors to the prompt in the training data.
To see the intuition behind this approximation, first consider the set of all subsets of $\spD$ of size $N$.
Suppose we pre-train an expert model on each of these subsets.
Then, at test-time, we select the expert model whose training data embeddings have the shortest distance to the prompt embedding.
Note that this is \emph{equivalent} to TTT, without the need for fine-tuning at test-time.
We have simply pre-computed all possible expert models at train-time, and at test-time fetch the right one.
TTMM approximates this interpretation of TTT in three steps.\looseness=-1

\paragraph{Approximation 1: Fewer amortized experts.}
The number of pre-trained expert models in the above scenario is not practically feasible, as it grows exponentially with the training data.
Thus, TTMM considers a more practical scenario with a smaller number of clusters where we do not pre-compute all possible TTT models.
Our following proposition indicates that even an expert model different to the TTT model, yet trained on a cluster $\spD'$ sufficiently close to the prompt embedding, may still be a good approximation of TTT:\looseness=-1

\begin{informalproposition}[Formalized in \cref{prop:approx_ttt}]\label{prop:approx_ttt:informal}
      Suppose the neural network $f(\vx; \vtheta)$ is $L$-Lipschitz in $\vtheta$ and that the loss is $G$-smooth in $\vx$, with any $L, G > 0$.
      Further, suppose we adapt models via $T \geq 1$ steps of gradient descent with step size $\eta > 0$ from a shared initialization~$\vtheta$.
      For any prompt $\vxs \in \spX$ and $N \geq 1$, let $\vtheta_{\vxs}$ be the TTT model trained on the set $\spD_{\vxs}$ of the $N$ nearest neighbors to~$\vphi(\vxs)$ in $\spD$.
      Let $\vtheta'$ be the expert model trained on any $\spD' \subseteq \spD$.
      Then, if $\spD'$ contains at least one nearest neighbor from $\spD_{\vxs}$, \begin{align}
        \norm{f(\vxs; \vtheta_{\vxs}) - f(\vxs; \vtheta')}
        \leq \eta T L G (\diam(\spD_{\vxs}) + \diam(\spD')).
    \end{align}
\end{informalproposition}

To understand the intuition behind this proposition, consider an expert model trained on a cluster $\spD'$ that contains parts of the local neighborhood of the prompt embedding (for this, the entire clustering must \emph{cover} the dataset).
Then, the above proposition shows that this expert model is a good approximation of the TTT model if the diameter of $\spD'$ is also~\emph{small}.
Consequently, this motivates the approximation of TTT by TTMM with many expert models \emph{covering} the entire dataset, each trained on a \emph{small} cluster.
Training expert models on this partitioning of the dataset does not cost more FLOPs or GPU memory than training a single model on the entire dataset.
This contrasts with TTT, which considers a separate expert model for each possible recombination of the training data into subsets of size $N$.\looseness=-1

\paragraph{Approximation 2: Summarizing experts with their centroids.}
Instead of selecting the expert model by minimizing the sum of distances of individual data points to the prompt embedding~(i.e.,~$\smash{\sum_{\vx \in \spD_k} \norm{\vphi(\vxs) - \vphi(\vx)}}$), TTMM summarizes each expert model by its centroid $\smash{\nicefrac{1}{\abs{\spD_k}} \sum_{\vx \in \spD_k} \vphi(\vx)}$.
Selecting the expert with the closest centroid to the prompt embedding further speeds up inference, since we only need to compute $K$ distances, rather than the distances to all training data embeddings.
This simplification does not come for free: Consider the one-dimensional set of embeddings~$\{-1,0,1\}$ and a prompt embedding of~$0$ with~$N = 2$.
Then, the optimal cluster is either~$\{-1,0\}$ or~$\{0,1\}$, yet TTMM would pick~$\{-1,1\}$ since its centroid averages out differences in the embeddings.
Since TTMM first clusters the data, there is reason to expect that each cluster has small diameter, in which case the centroid is a good approximation of the individual data points.\looseness=-1

\paragraph{Approximation 3: Merging multiple experts.}
As a final step, to compensate for the approximation error of the chosen expert model induced by steps (1) and (2), TTMM merges multiple expert models.
Intuitively, and as visualized in \cref{fig:ttmm}, this is motivated by multiple experts being able to cover distinct features of the prompt that may not all be captured by a single expert model.
We find empirically that merging a small number of experts can improve model performance without incurring a significant test-time overhead~(cf.~\cref{fig:results}).
TTMM computes merging coefficients via cross-attention, which is equivalent to weighting the active experts with an RBF kernel:\footnote{See \cref{sec:cross_attention_vs_rbf} for a proof.} \[
  \softmax\parentheses*{\frac{1}{\beta} \vphi_k^\top \vphi^\star} = \mathrm{normalize}\parentheses*{\exp\parentheses*{-\frac{1}{2\beta} \norm{\vphi^\star - \vphi_k}_2^2}}.
\]
Weighting data (not models!) by an RBF kernel has long been a common approach in machine learning and used, e.g., in kernel density estimation~\citep{rosenblatt1956remarks,parzen1962estimation}, kernel regression~\citep{nadaraya1964estimating,watson1964smooth}, RBF networks~\citep{lowe1988multivariable,moody1989fast}, and local learning~\citep[e.g.,][]{bottou1992local,atkeson1997locally}.
TTMM additionally sparsifies the selected models to minimize the test-time compute overhead due to loading and merging.\looseness=-1

\section{Results}

\begin{table}
  \vspace{-1ex}
  \centering
  \renewcommand{\arraystretch}{1.25}
  \setlength{\tabcolsep}{8pt}
  \begin{tabular}{lcccc}
      \toprule
      \multirow{2}{*}{\textbf{Model}} & \multicolumn{2}{c}{\textbf{Wikipedia}} & \multicolumn{2}{c}{\textbf{GitHub (Python)}} \\
      \cmidrule(lr){2-3} \cmidrule(lr){4-5}
      & \textit{Merging} & \textcolor{gray}{\textit{Ensembling}} & \textit{Merging} & \textcolor{gray}{\textit{Ensembling}} \\
      \midrule
      \textit{Llama-3.2-1B} & \multicolumn{2}{c}{$8.674$} & \multicolumn{2}{c}{$2.611$} \\
      \hdashline
      \quad + fine-tuning & \multicolumn{2}{c}{$7.849$} & \multicolumn{2}{c}{$2.581$} \\
      \quad + TTMM (1 active expert) & \multicolumn{2}{c}{$7.669$} & \multicolumn{2}{c}{$2.552$} \\
      \quad + TTMM (3 active experts) & 7.571 & \textcolor{gray}{7.579} & 2.519 & \textcolor{gray}{2.512} \\
      \rowcolor{lightgray}
      \quad + \textbf{TTMM} (10 active experts) & \underline{\textbf{7.510}} & \textcolor{gray}{\underline{7.509}} & \textbf{2.492} & \textcolor{gray}{2.464} \\
      \hdashline
      \quad \textcolor{gray}{+ TTT} & \multicolumn{2}{c}{\textcolor{gray}{${7.559}$}} & \multicolumn{2}{c}{\textcolor{gray}{$\underline{2.441}$}} \\
      \midrule
      \textit{Qwen2.5-1.5B} & \multicolumn{2}{c}{8.570} & \multicolumn{2}{c}{2.335} \\
      \hdashline
      \quad + fine-tuning & \multicolumn{2}{c}{7.702} & \multicolumn{2}{c}{2.317} \\
      \quad + TTMM (1 active expert) & \multicolumn{2}{c}{7.166} & \multicolumn{2}{c}{2.330} \\
      \quad + TTMM (3 active experts) & 7.101 & \textcolor{gray}{7.096} & 2.304 & \textcolor{gray}{2.294} \\
      \rowcolor{lightgray}
      \quad + \textbf{TTMM} (10 active experts) & \underline{\textbf{7.080}} & \textcolor{gray}{\underline{7.061}} & \textbf{2.287} & \textcolor{gray}{2.257} \\
      \hdashline
      \quad \textcolor{gray}{+ TTT} & \multicolumn{2}{c}{\textcolor{gray}{7.231}} & \multicolumn{2}{c}{\textcolor{gray}{\underline{2.177}}} \\
      \bottomrule
  \end{tabular}
  \caption{Perplexity (lower is better) across evaluation datasets and base models with $100$~experts. \textbf{Bold}~numbers denote the best test-time efficient method, and \underline{underlined}~numbers denote the best overall. TTMM outperforms standard multitask fine-tuning of a single model.\looseness=-1}
  \label{table:main}
\end{table}

We evaluate language modeling with causal language models.
We report the perplexity of the models on holdout data from the Wikipedia~\citep[English,][]{wikidump} and GitHub~(Python code) corpora, comprising 6.4M and 7.2M documents, respectively.
As base models, we use Llama-3.2-1B~\citep{grattafiori2024llama} and Qwen2.5-1.5B~\citep{yang2024qwen2}.
For clustering and model selection, we use embeddings from the \texttt{all-mpnet-base-v2} model~\citep{reimers2019sentence,song2020mpnet}.
Following \cite{hardt2023test} and \cite{huebotter2025efficiently}, we evaluate TTT by training a LoRA adapter for a single gradient step each on the $100$ nearest neighbors to the prompt embedding, most to least similar.
We also compare against a single model, which is fine-tuned for a single epoch on the respective corpus.
This fine-tuned model can be considered an upper bound on the performance of approaches to multitask learning, since multitask learning typically trains one model per meta-task (i.e., per corpus) and merging models from multiple meta-tasks typically reduces performance on any individual meta-task.\looseness=-1

Our code is available at \url{https://github.com/rbertolissi/ttmerge}.
We share fine-tuned baselines and TTMM-MoEs on \href{https://huggingface.co/collections/rbertolissi/test-time-model-merging-ttmm-6886dec2c436cc4ceaf39ff7}{Huggingface}.\looseness=-1 %

\paragraph{\insight TTMM approaches the accuracy of TTT at almost no test-time overhead.}
We show in \cref{fig:results} and \cref{table:main} that TTMM with $10$ active experts achieves close to the accuracy of TTT, and much higher accuracy than a single model that was fine-tuned on the respective corpus.
Instead of merging experts in parameter space, we also evaluate merging experts in prediction space by ensembling the experts.
We find that this yields marginally higher accuracy, yet it has a significantly larger test-time overhead.
Overall, merging experts in parameter space achieves the best trade-off between accuracy and test-time efficiency.
Interestingly, TTMM outperforms TTT on the Wikipedia dataset. This may be because TTMM can incorporate knowledge from a larger set of experts, each trained on a few thousands of data points. In contrast, TTT is only training on the 100 most related neighbors, which may miss some information if there are more than 100 related data points. In our view, the empirical difference between TTMM and TTT is an interesting direction for further study.\looseness=-1

\begin{figure}
  \vspace{-2ex}
  \centering
  \incplt[\textwidth]{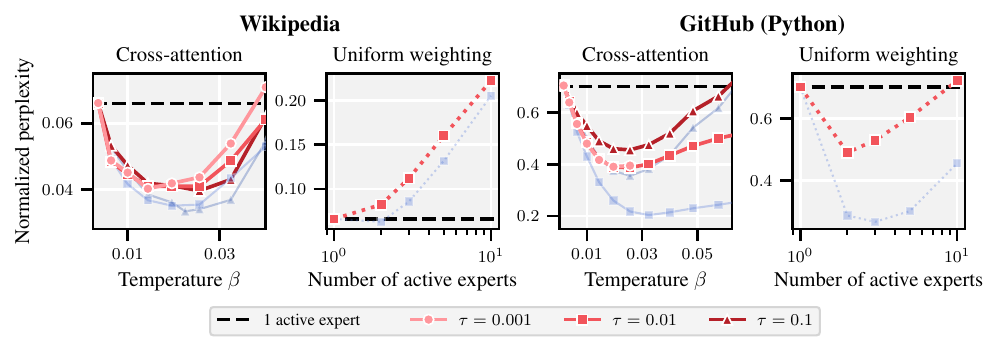}
  \vspace{-2.5ex}
  \caption{
    On each dataset, we evaluate the performance improvement of TTMM's cross-attention mechanism compared to uniform weighting of active experts.
    Uniform weighting (dotted) selects the experts closest to the prompt embedding, and weights them equally.
    We see across merging (\textcolor{red}{red}) and ensembling (\textcolor{blue}{light blue}) that the cross-attention mechanism outperforms uniform weighting significantly.
    Notably, on Wikipedia, combining multiple experts helps only when using cross-attention.
    The number of experts is $K = 1000$.\looseness=-1}
  \label{fig:merging_coefficient}
  \vspace{-1ex}
\end{figure}

\paragraph{\insight Cross-attention is more effective than other merging coefficients.}
In \cref{fig:merging_coefficient}, we evaluate perplexity with merging coefficients computed via cross-attention to uniform weighting of active experts.
We find that cross-attention significantly outperforms uniform weighting.
In particular, on the Wikipedia dataset, merging multiple experts does not help at all when weighting them uniformly.
We also compare against merging coefficients, such as leveraging the logit-entropy of the experts~(cf.~\cref{sec:alternative_merging}), but we do not find them to outperform TTMM's cross-attention while being significantly more computationally expensive.
We find that sparsity $\tau=0.01$ provides the best trade-off between accuracy and test-time efficiency, and we fix it in all other experiments.
Furthermore, on our evaluation corpora, we need to select roughly 10 active experts to cover the local neighborhood of the prompt~(cf.~\figref{fig:clusters}{right}), which is also when the performance of TTMM begins to saturate.\looseness=-1

\NewDocumentCommand{\incpltcluster}{O{\columnwidth}m}{%
  \begin{center}
    \adjustbox{center}{\adjustbox{width=#1+10pt}{
      \includegraphics[width=0.28\textwidth]{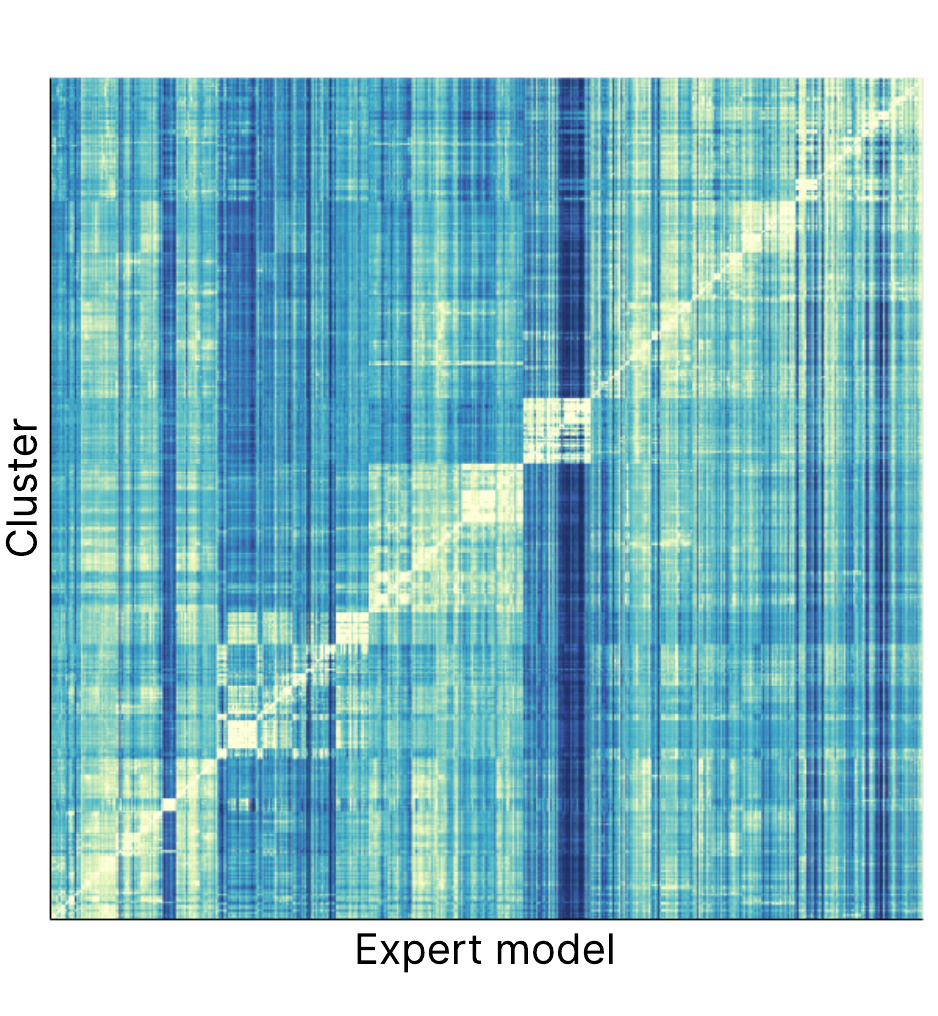}
      \hfill
      \includegraphics[width=0.6\textwidth]{./plots/output/#2.pdf}
    }}
  \end{center}
}
\begin{figure}
  \vspace{-2.5ex}
  \incpltcluster[\textwidth]{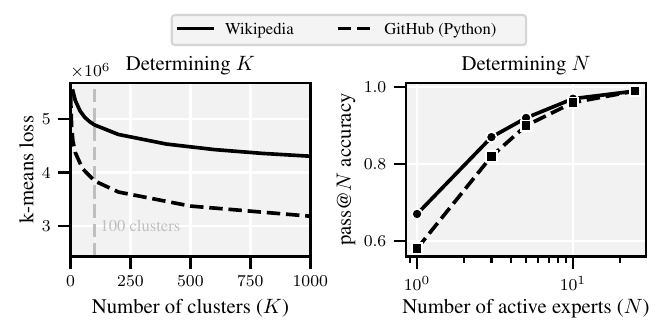}
  \vspace{-2.5ex}
  \caption{
    \textbf{Left:}~Perplexity~(brighter is better) of $1000$ experts on clusters of the Wikipedia dataset. The lowest perplexity is achieved on the diagonal, i.e., when the expert model is evaluated on holdout data from the cluster it was trained on. The block-diagonal structure is an artifact of bisecting k-means, which leads to similar neighboring clusters.
    \textbf{Middle:}~In our evaluation datasets, the rate of reduction in k-means loss is largest until around $100$ clusters.
    \textbf{Right:}~For 100 random samples from each cluster's holdout set, we measure whether the correct cluster is among the top-$N$ closest clusters and report the average pass@$N$ accuracy.\looseness=-1}
  \label{fig:clusters}
\end{figure}

\paragraph{\insight Local experts outperform the base model in their local neighborhood.}
We find in \figref{fig:clusters}{left} that expert models perform significantly better on holdout data in their cluster than the base model or other expert models.
This reaffirms the earlier findings of \cite{hardt2023test} that TTT on local neighborhoods improves accuracy.
We also find a large qualitative difference between the text generated by different experts and the base model, and we include examples in \cref{sec:qualitative_examples}.\looseness=-1

\paragraph{Choosing the number of (active) experts.}

We find, as summarized in \cref{fig:results}, that TTMM with $100$ experts significantly outperforms TTMM with $10$ experts or the multitask fine-tuned model (i.e., a single expert).
This indicates that scaling the number of experts beyond the number of meta-tasks can increase accuracy substantially.
Nevertheless, we also find that scaling the number of experts to $1000$ does not further improve accuracy on our evaluation corpora.
In the following, we discuss considerations for the number of experts in TTMM:
\begin{itemize}\vspace{-3pt}
    \item \textbf{Number of experts ($K$):} The total number of experts has almost no direct effect on latency. However, more experts require a large amount of static CPU memory. Further, increasing the number of experts also increases the separation of knowledge/skills between experts. This then requires more active experts to retain the same performance, which also increases latency and may lead to model interference due to merging. To prevent knowledge fragmentation, $K$ should be determined depending on the training set. In our experiments, choosing 100 experts is indicated by the ``elbow method'', common in clustering~(cf.~\figref{fig:clusters}{middle}).\looseness=-1

    \item \textbf{Number of active experts ($N$):}
    The number of active experts has the strongest effect on latency, since all active experts have to be loaded from CPU into GPU memory before inference. Our ablation in \figref{fig:clusters}{right} indicates that with $10$ active experts, the ``correct'' expert is likely to be active, while latency remains minimal.\looseness=-1
\end{itemize}\vspace{-3pt}

\paragraph{Evaluation on MMLU.}

We additionally evaluate TTMM on MMLU~\citep{hendrycks2021measuring} as a non-perplexity task.
We use (non-instruct) Llama-3.2-1B as base model, which we then fine-tune for one epoch on the MMLU training set.
This fine-tuned model achieves approximately the same performance as instruction-tuned Llama-3.2-1B.
We use the non-instruct model to avoid any confounding effect of RL-training or instruction-tuning on privileged data from the instruct model. As embedding model, we use Stella (\texttt{stella\_en\_400M\_v5}) but obtain similar results with MPNet~(cf.~\cref{sec:embedding_model_ablation}). We train 100 experts on the MMLU training set, using the globally fine-tuned model as initialization.
We summarize the results in \cref{table:mmlu} and find that TTMM slightly increases accuracy.
A broader evaluation on other non-perplexity tasks and non-language domains is an important direction for future work.\looseness=-1

\begin{table}
  \centering
  \renewcommand{\arraystretch}{1.25}
  \setlength{\tabcolsep}{8pt}
  \begin{tabular}{lcccccc}
    \toprule
    \textbf{Model} & \textbf{Humanities} & \textbf{Social Sciences} & \textbf{STEM} & \textbf{Other} & \textbf{Overall} \\
    \midrule
    Fine-tuned            & 44.8   & 54.24 & 40.79 & 54.43 & 48.10 \\
    TTMM (1 expert)       & 45.25  & 54.92 & 40.98 & 54.30 & 48.41 \\
    TTMM (3 experts)      & 45.33  & 55.02 & 40.98 & 54.39 & 48.48 \\
    TTMM (10 experts)     & \textbf{45.63} & 55.15 & 41.39 & 54.55 & 48.74 \\
    TTMM (15 experts)     & 45.46  & \textbf{55.31} & \textbf{41.77} & \textbf{55.26} & \textbf{48.96} \\
    \bottomrule
  \end{tabular}
  \caption{Accuracy (in \%) of TTMM on MMLU. \textbf{Bold} numbers denote the highest accuracy.}
  \label{table:mmlu}
\end{table}

\paragraph{Ablations.}

We perform a series of ablations, which we summarize in \cref{sec:ablations}.
Notably, we find that (1) larger embedding models need not lead to improved performance of TTMM and that (2) TTMM performs similarly when fixing the number of experts as to when selecting the number of active experts dynamically via thresholding.\looseness=-1

\section{Discussion and Future Work}

We propose \emph{\textbf{T}est-\textbf{T}ime \textbf{M}odel \textbf{M}erging}~(TTMM) which scales the mixture of experts paradigm to orders of magnitude more experts than in state-of-the-art multitask methods, leading to a substantial improvement in language modeling ability.
We find that TTMM can achieve performance gain close to TTT by merging \emph{few} active experts into a single model per query.
TTMM achieves this with almost the same test-time cost as inference with the base model, and without increasing train-time FLOPs or GPU memory compared to multitask fine-tuning.
A limitation of TTMM is the additional CPU memory required for storing all experts.\looseness=-1

In achieving close to the performance of TTT without nearly any test-time overhead, TTMM opens up several exciting directions for future research.
First, TTMM makes it feasible to re-select models online during generation or evaluation based on the growing prefix, which with TTT, is too expensive in most applications.
\cite{gururangan2023scaling} have already shown that online re-selection can work with an ensemble selected from a smaller number of experts.
Second, one may further improve the performance of TTMM using tools from the literature on model merging and MoEs aimed at avoiding interference between experts and improving the routing of data to experts, e.g., by fine-tuning the router on holdout data.
TTMM can further be applied to settings with privileged information by adding or removing experts dynamically based on user access.
Furthermore, it may be possible to leverage the large number of experts in TTMM for uncertainty estimation, e.g., by tracking also the variance of the merged parameters.
Finally, many reasoning language models~(called RLMs) such as DeepSeek~R1~\citep{guo2025deepseek} are based on a MoE architecture, and TTT has been shown to be able to improve RLMs~\citep{zuo2025ttrl}.
Since TTMM can be considered a scalable MoE architecture that approximates TTT, it would be particularly interesting to study whether TTMM can improve reasoning models.\looseness=-1

\section*{Acknowledgments}

This project was supported in part by the European Research Council (ERC) under the European Union's Horizon 2020 research and Innovation Program Grant agreement no. 815943, and the Swiss National Science Foundation under NCCR Automation, grant agreement 51NF40 180545. Ido Hakimi was supported by an ETH AI Center Postdoctoral fellowship.\looseness=-1

\bibliography{colm2025_conference}
\bibliographystyle{colm2025_conference}

\clearpage\appendix

\section*{\LARGE Appendices}

\startcontents
\printcontents{}{0}[2]{}

\clearpage

\begin{figure}[H]
\begin{mdframed}[
    backgroundcolor=lightgray,
    linecolor=lightgray,
    linewidth=0pt,
    roundcorner=5pt,
    innertopmargin=8pt,
    innerbottommargin=8pt,
    innerleftmargin=8pt,
    innerrightmargin=8pt
]
    \begin{minted}[
        linenos,
        numbersep=4pt,
        xleftmargin=0.7em,
        frame=none,
        fontsize=\small
    ]{python}
for name in lora_param_names:
  As = torch.stack([
    d[name]["A"] for d in lora_dicts])
  Bs = torch.stack([
    d[name]["B"] for d in lora_dicts])
  delta = torch.einsum(
    "k i r, k r o -> i o", Bs, As)
    \end{minted}
\end{mdframed}
\caption{We merge LoRA adapters using PyTorch's \texttt{einsum}, before applying the low-rank updates to the model parameters.}
\label{fig:merging_code}
\end{figure}

\section{Additional Related Work}\label{sec:additional_related_work}

\paragraph{Retrieval-Augmented Generation.}
An alternative to specializing model parameters at test-time (i.e., TTT) is to instead specialize predictions by conditioning an autoregressive model directly on prompt-related data.
This is called retrieval-augmented generation~\citep[RAG,][]{lewis2020retrieval,khandelwal2019generalization,guu2020retrieval,borgeaud2022improving} and motivated by the ability of LLMs to learn from context~\citep{brown2020language,wei2022chain}.
Similarly to TTT, RAG leads to expensive inference since all local data needs to be processed at test-time.
TTMM is more efficient by compressing local information into expert LoRA adapters at train-time.
RAG is commonly used to endow LLMs with privileged information.
Similarly, TTMM can include (\& remove) information simply by adding (or removing) local experts.\looseness=-1

\section{Equivalence of Cross-Attention and RBF Kernel}\label{sec:cross_attention_vs_rbf}

Assuming that all embeddings are normalized, we have \begin{align*}
    \exp\parentheses*{-\frac{1}{2 \beta} \norm{\vphi^\star - \vphi_k}_2^2} &= \exp\parentheses*{-\frac{1}{\beta} + \frac{1}{\beta} \vphi_k^\top \vphi^\star} \propto \exp\parentheses*{\frac{1}{\beta} \vphi_k^\top \vphi^\star}.
\end{align*}

\section{Proof of \cref{prop:approx_ttt:informal}}\label{sec:proofs}

\paragraph{Setup.}
Let $f(\vx;\vtheta)$ be a neural network parameterized by $\vtheta\in\R^{p}$ and evaluated on input~${\vx\in\spX}$.
We have a loss function $\spL(\vtheta;\vz) \in \R$ for data point $\vz=(\vx,y) \in \spX \times \spY$.\footnote{This supervised data may be constructed via self-supervision, as in language modeling.}
Let $\spD \subseteq \spX \times \spY$ be a dataset of input-output pairs.
Given a prompt~${\vxs \in \spX}$ and any $N \geq 1$, we consider the set of $N$ nearest neighbors to $\vphi(\vxs)$ in $\spD$ (with respect to the input embedding): \begin{align*}
  \spD_{\vxs} = \{\vz_1,\dots,\vz_N\} \subseteq \spD.
\end{align*}
Let $\spD'$ be any dataset of size $M \geq 1$:
\[
  \spD' = \{\vz_1',\dots,\vz_M'\} \subseteq \spD.
\]
We first analyze single-step gradient descent with step size $\eta>0$ from the same initialization~$\vtheta$:
\[
  \vtheta_{\vxs} = \vtheta - \eta \left[
    \frac{1}{N} \sum_{i=1}^N \grad_\vtheta \spL(\vtheta;\,\vz_i)
  \right],
  \quad
  \vtheta' = \vtheta - \eta \left[
    \frac{1}{M} \sum_{j=1}^M \grad_\vtheta \spL(\vtheta;\,\vz_j')
  \right].
\]

\paragraph{Assumptions.}

\begin{enumerate}
    \item[(A1)] \textbf{Lipschitz-smoothness of the loss.}
    There is a constant $G>0$ such that for all $\vz=(\vx,y),\vz' = (\vx',y')$ and all $\vtheta$:
    \[
        \bigl\|\grad_\vtheta \spL(\vtheta;\,\vz)\;-\;\grad_\vtheta \spL(\vtheta;\,\vz')\bigr\|
        \;\le\;
        G\,\|\vx - \vx'\|.
    \]

    \item[(A2)] \textbf{Network Lipschitz in parameters.}
    There is a constant $L>0$ such that for all $\vtheta,\vtheta'$ and all $\vx$:
    \[
        \bigl\|f(\vx;\vtheta)\;-\;f(\vx;\vtheta')\bigr\|
        \;\le\;
        L\,\|\vtheta - \vtheta'\|.
    \]
\end{enumerate}\vspace{\baselineskip}

\begin{proposition}\label{prop:approx_ttt}
  Suppose (A1) and (A2) hold, and let $\vtheta$ be any initialization.
  For any prompt ${\vxs \in \spX}$ and $N \geq 1$, let $\vtheta_{\vxs}$ be the model trained via single-step gradient descent with step size~${\eta > 0}$ on $\spD_{\vxs}$.
  Let $\vtheta'$ be the expert model trained on any $\spD' \subseteq \spD$.
  Then, if $\spD' \cap \spD_{\vxs} \neq \emptyset$, \begin{align}
    \norm{f(\vxs; \vtheta_{\vxs}) - f(\vxs; \vtheta')}
    \leq \eta L G (\diam(\spD_{\vxs}) + \diam(\spD')).
\end{align}
\end{proposition}

\begin{proof}
  We have \begin{align*}
    \vtheta_{\vxs} - \vtheta' &= - \eta \brackets*{\frac{1}{N} \sum_{i=1}^N \grad_\vtheta \spL(\vtheta; \vz_i) - \frac{1}{M} \sum_{j=1}^M \grad_\vtheta \spL(\vtheta; \vz_j')} \\
    &= - \eta \brackets*{\frac{1}{M} \sum_{j=1}^M \frac{1}{N} \sum_{i=1}^N \grad_\vtheta \spL(\vtheta; \vz_i) - \frac{1}{M} \sum_{j=1}^M \grad_\vtheta \spL(\vtheta; \vz_j')} \\
    &= - \eta \brackets*{\frac{1}{M} \sum_{j=1}^M \parentheses*{\frac{1}{N} \sum_{i=1}^N \grad_\vtheta \spL(\vtheta; \vz_i) - \grad_\vtheta \spL(\vtheta; \vz_j')}}.
  \end{align*}
  Thus, \begin{align*}
    \norm*{\vtheta_{\vxs} - \vtheta'} &\leq \eta \frac{1}{M} \sum_{j=1}^M \frac{1}{N} \sum_{i=1}^N \norm*{\grad_\vtheta \spL(\vtheta; \vz_i) - \grad_\vtheta \spL(\vtheta; \vz_j')} \tag{triangle inequality} \\
    &\leq \eta G \frac{1}{M} \sum_{j=1}^M \frac{1}{N} \sum_{i=1}^N \norm*{\vx_i - \vx_j'}. \tag{Assumption A1} \\
    \intertext{Let $\vx$ be any point in $\spD' \cap \spD_{\vxs}$. We have}
    &= \eta G \frac{1}{M} \sum_{j=1}^M \frac{1}{N} \sum_{i=1}^N \norm*{\vx_i - \vx + \vx - \vx_j'} \\
    &\leq \eta G \frac{1}{M} \sum_{j=1}^M \frac{1}{N} \sum_{i=1}^N \parentheses*{\norm*{\vx_i - \vx} + \norm*{\vx - \vx_j'}} \tag{triangle inequality} \\
    &\leq \eta G (\diam(\spD_{\vxs}) + \diam(\spD')).
  \end{align*}
  Finally, \begin{align*}
    \norm*{f(\vxs; \vtheta_{\vxs}) - f(\vxs; \vtheta')} &\leq L \norm*{\vtheta_{\vxs} - \vtheta'} \tag{Assumption A2} \\
    &\leq \eta L G (\diam(\spD_{\vxs}) + \diam(\spD')).
  \end{align*}
\end{proof}

To complete the proof of \cref{prop:approx_ttt:informal}, we simply observe that the bound $\norm*{\vtheta_{\vxs}^{(1)} - \vtheta'^{(1)}} \leq \eta G (\diam(\spD_{\vxs}) + \diam(\spD'))$ can be applied recursively from the shared initialization $\vtheta^{(0)}$ for $T$ steps of gradient descent via the relation \begin{align*}
    \norm*{\vtheta_{\vxs}^{(t+1)} - \vtheta'^{(t+1)}} \leq \norm*{\vtheta_{\vxs}^{(t)} - \vtheta'^{(t)}} + \eta G (\diam(\spD_{\vxs}) + \diam(\spD')).
\end{align*}
Unrolling this recursion proves \cref{prop:approx_ttt:informal} with an additional factor $T$ compared to \cref{prop:approx_ttt}.

\section{Ablations}\label{sec:ablations}

\subsection{Fixing the Number of Active Experts}

In \cref{table:ttmm_vs_fixed}, we ablate the choice of selecting the number of active experts dynamically per prompt depending on the threshold $\tau$~(cf.~\cref{alg:ttmm_test_time}) rather than specifying a fixed sparsity $N$.
On our evaluated corpora, we find no significant difference in performance.\looseness=-1

\begin{table}
  \centering
  \renewcommand{\arraystretch}{1.25}
  \setlength{\tabcolsep}{8pt}
  \begin{tabular}{llcccc}
    \toprule
    \textbf{Model} & \textbf{\# Active Experts} & \multicolumn{2}{c}{\textbf{Wikipedia}} & \multicolumn{2}{c}{\textbf{GitHub (Python)}} \\
    \cmidrule(lr){3-4} \cmidrule(lr){5-6}
     & & TTMM & Fixed & TTMM & Fixed \\
    \midrule
    \multirow{3}{*}{\textit{Llama-3.2-1B}}
      & 1 & 7.669 & 7.669 & 2.552 & 2.552 \\
      & 3 & 7.571 & 7.580 & 2.519 & 2.505 \\
      & 10 & {7.510} & {7.509} & {2.492} & {2.491} \\
    \midrule
    \multirow{3}{*}{\textit{Qwen2.5-1.5B}}
      & 1 & 7.166 & 7.166 & 2.330 & 2.331 \\
      & 3 & 7.101 & 7.104 & 2.304 & 2.295 \\
      & 10 & {{7.080}} & 7.080 & {{2.287}} & 2.287 \\
    \bottomrule
  \end{tabular}
  \caption{Perplexity comparison (lower is better) between TTMM (dynamic expert routing) and fixed expert selection for Llama and Qwen across different numbers of active experts.}
  \label{table:ttmm_vs_fixed}
\end{table}

\subsection{Different Embedding Models}\label{sec:embedding_model_ablation}

In \cref{tab:embedding_models}, we evaluate different embedding models on MMLU and find that TTMM outperforms the base model regardless of the choice of embedding model.
Notably, the largest Qwen2-1.5B embedding model performs worse than the 100M MPNet model. A possible reason is that finding a decent clustering might be significantly easier than directly solving questions, since the former only requires a broad understanding of the meaning of some keywords while the latter requires some ability of reasoning.
Stella~\citep{zhang2024jasper} performs the best among the tested embedding models.\looseness=-1

\begin{table}
  \centering
  \small
  \renewcommand{\arraystretch}{1.25}
  \setlength{\tabcolsep}{8pt}
  \begin{tabular}{lccccc}
    \toprule
    \textbf{Model} & \textbf{Humanities} & \textbf{Social Sciences} & \textbf{STEM} & \textbf{Other} & \textbf{Overall} \\
    \midrule
    Fine-tuned                               & 44.80 & 54.24 & 40.79 & 54.43 & 48.10 \\
    \texttt{all-mpnet-base-v2}        & \textbf{45.95} & 54.76 & 40.56 & 54.72 & 48.61 \\
    \texttt{stella\_en\_400M\_v5}      & 45.46 & \textbf{55.31} & \textbf{41.77} & \textbf{55.26} & \textbf{48.96} \\
    \texttt{gte-Qwen2-1.5B-instruct}   & 45.27 & 54.76 & 40.91 & 54.65 & 48.45 \\
    \bottomrule
  \end{tabular}
  \caption{MMLU accuracy (in \%) comparing TTMM with different embedding models against the fine-tuned baseline. \textbf{Bold} indicates the highest accuracy.}
  \label{tab:embedding_models}
\end{table}

\subsection{Evaluation of Alternative Merging Coefficients}\label{sec:alternative_merging}

\paragraph{SIFT: Accounting for redundancy between experts.}

We evaluate weighting active experts according to their estimated uncertainty reduction about the response to the prompt, as proposed by~\cite{huebotter2025efficiently}.
Building on transductive active learning~\citep{hubotter2024transductive}, \cite{huebotter2025efficiently} approximate the expected reduction in uncertainty about the prompt when conditioning on the $i$-th data point by $\smash{\sigma_{i-1}^2 - \sigma_{i}^2}$ where $\smash{\sigma_i^2}$ estimates the uncertainty after seeing the $i$-th data point.\looseness=-1

We use the same uncertainty estimates $\smash{\sigma_i^2}$ to weight the experts in TTMM, but instead of conditioning on individual data points represented by their embeddings, we condition on expert models represented by their centroids.
For a given prompt, we first retrieve the $N$ nearest centroids and then reweight them according to: \begin{align*}
  w_k \gets \frac{\sigma_{k-1}^2 - \sigma_k^2}{\sum_{k'=1}^{N} \sigma_{k'-1}^2 - \sigma_{k'}^2} = \frac{\sigma_{k-1}^2 - \sigma_k^2}{\sigma_0^2 - \sigma_{N}^2}.
\end{align*}
Given that embeddings are normalized, we have $\smash{\sigma_0^2 = 1}$.\looseness=-1

Intuitively, these weights account for redundancy between clusters.
Suppose we duplicate a cluster $m$ times, then the above weights downweigh each of the $m$ duplicates by a factor $\nicefrac{1}{m}$.
The cross-attention mechanism in TTMM, on the other hand, would treat each duplicate cluster as if it were a distinct cluster, and hence increase the weight of the duplicated cluster by $m$.\looseness=-1

\begin{figure}
  \incplt[\textwidth]{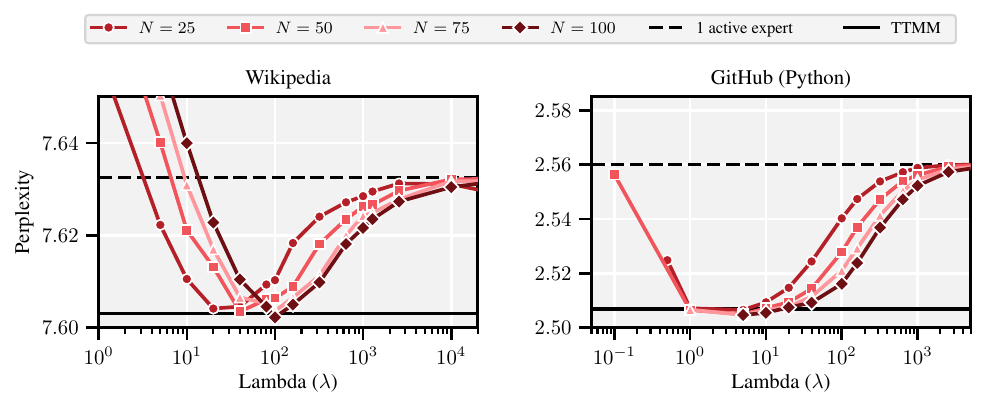}
  \vspace{-2.5ex}
  \caption{
    We evaluate weighting experts according to SIFT's uncertainty estimates, which take into account redundancy between clusters.
    Here, $\lambda$ is the hyperparameter of SIFT. As $\lambda \to \infty$, SIFT selects only the centroid closest to the prompt embedding, and as $\lambda \to 0$, SIFT weights all clusters uniformly.\looseness=-1}
  \label{fig:sift}
\end{figure}

We evaluate SIFT in \cref{fig:sift}.
We find that when we have control over the created experts (e.g., by clustering the data) as in TTMM, SIFT yields only a negligible improvmenet over the cross-attention mechanism.
This likely stems from the fact that the centroids of the clusters are already well-separated, and hence the cross-attention mechanism is able to effectively weight the experts.
We expect that in settings where we have less control over the experts, SIFT-weighting could provide a more significant improvement.\looseness=-1

\paragraph{DaWin: Weighting experts by their logit-entropy.}

\cite{oh2024dawin} propose DaWin, which weights experts by their logit-entropy.
Let $p_k(\vxs)$ be the distribution over logits of expert $k$ for input $\vxs$ with $H(p_k(\vxs))$ denoting its entropy.
Then, the weight of expert $k$ is given by \begin{align*}
  w_k \gets \ssoftmax_\tau\parentheses*{-\frac{1}{\beta} H(p_k(\vxs))}.
\end{align*}
Note that these weights are equivalent to the weights obtained by the cross-attention mechanism in TTMM, up to the use of $\smash{- H(p_k(\vxs))}$ to measure similarity as opposed to $\smash{\vphi_k^\top \vphi(\vxs)}$.
As opposed to the inner product, which can be evaluated with very little overhead, computing the entropy requires a separate forward pass with each expert.
This makes it as expensive as ensembling experts at test-time.
In \cref{fig:dawin}, we compare the logit-entropy weighting to the inner product weighting of TTMM, and do not find a performance difference.\looseness=-1

\begin{figure}
  \incplt[0.5\textwidth]{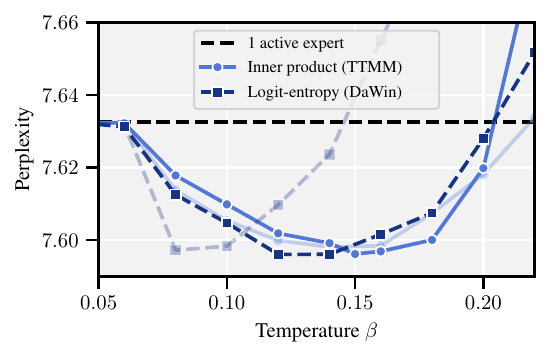}
  \vspace{-2.5ex}
  \caption{
    We evaluate ensembling models with inner product weighting (like in TTMM) and logit-entropy weighting (like in DaWin). We evaluate on Wikipedia with Llama-3.2-1B as base model. Solid lines are with sparsity $\tau = 0.1$, translucent lines are with $\tau = 0.01$.\looseness=-1}
  \label{fig:dawin}
\end{figure}

\paragraph{Ablation against summarizing each expert by its clusters' centroid.}

As suggested in \cref{sec:approximation}, we evaluate whether summarizing each expert by its clusters' centroid meaningfully degenerates the performance of TTMM.
We find that this appears not to be the case, as shown in \cref{table:centroid_ablation}.\looseness=-1

\begin{table}
  \centering
  \begin{tabular}{llr}
    \toprule
    & \textbf{Selection criterion} & \textbf{Perplexity (lower is better)} \\[1pt]
    \hline \\[-6pt]
    Summarized by centroid & $\norm{\vphi^\star - \vphi_k}^2$ & 7.669 \\[2pt]
    Not summarized by centroid & $\sum_{\vx \in \spD_k} \norm{\vphi^\star - \vphi(\vx)}^2$ & 7.633 \\[2pt]
    \bottomrule
  \end{tabular}
  \caption{We evaluate the selection of experts, without approximating each expert by its clusters' centroid as suggested in \cref{sec:approximation}. The evaluation is on the Wikipedia corpus with Llama-3.2-1B as base model, $K=1000$ total experts and a single active expert.}
  \label{table:centroid_ablation}
\end{table}

\section{Experiment Details}

\begin{table}[H]
  \centering
  \begin{tabular}{lr}
    \toprule
    \textbf{Parameter} & \textbf{Value} \\[1pt]
    \hline \\[-6pt]
    Number of clusters~$K$ & 100 \\[2pt]
    Learning rate & 2e-4 \\[2pt]
    LoRA rank & 64 \\[2pt]
    LoRA alpha & 16 \\[2pt]
    Adam epsilon & 1e-8 \\[2pt]
    Adam beta & (0.9, 0.999) \\[2pt]
    Weight decay & 0.01 \\[2pt]
    Batch size & 4 \\[2pt]
    \hline \\[-7pt]
    Sparsity $\tau$ & 0.01 \\[2pt]
    \bottomrule
  \end{tabular}
  \caption{Hyperparameters.}
  \label{table:hyperparams}
\end{table}

\begin{table}[H]
  \centering
  \small
  \begin{tabular}{lr}
    \toprule
    \textbf{Dataset} & \textbf{Source} \\[1pt]
    \hline \\[-6pt]
    Wikipedia & \url{https://huggingface.co/datasets/wikimedia/wikipedia} \\[2pt]
    GitHub (filtered to Python) & \url{https://huggingface.co/datasets/codeparrot/github-code-clean} \\[2pt]
    \bottomrule
  \end{tabular}
  \caption{Datasets.}
  \label{table:datasets}
\end{table}

All models were trained on the first 1024 tokens of each training example. Testing was conducted on the full input length, with the context size limited to $\smash{2^{14}}$ tokens. For each dataset, the test sets were constructed by selecting a single random example from each cluster's holdout set, ensuring that the evaluation approximately covers the data of all clusters.\looseness=-1

The hyperparameters for TTT match those used to train the expert models, except for a learning rate of 5e-4. The models were trained for five epochs and the best checkpoint was selected for evaluation.\looseness=-1

Unless noted otherwise, all ablations are with Llama-3.2-1B as the base model.
Further, unless noted otherwise, we use ``number of active experts'' as an upper bound to the mean number of active experts across test examples.\looseness=-1

\paragraph{Dataset size for expert training.}
Both the Wikipedia and the GitHub (Python) corpora contain roughly 7M documents.
Therefore, since each train document is at most 1024 tokens long, each expert is trained on at most $\approx$7M tokens if we train $1000$ experts.
This is a relatively small dataset size for training experts, compared to datasets used for training common MoE or multitask models.

\paragraph{Preventing Information Leakage.}

The TTT setting (and also TTMM) requires specification of a prompt to obtain a task-specific model.
In language modeling datasets that do not have the natural structure of prompt and response, it would be unfair to use the response as prompt.
Using the response as prompt would invalidate the evaluation, as the test data is leaked to the model before its evaluation.
To prevent this, we split individual data points into a prefix (prompt) and suffix (response) similar to \cite{hardt2023test}, use the prefix for obtaining the task-specific model, and evaluate the model on the suffix.
\Cref{table:prefix_token_count} summarizes the choice of prefix per dataset.
When not dynamically re-selecting models before the generation of each token (or after every $N$ tokens), the prefix must be sufficiently informative for the suffix.
On Wikipedia, we find this to be around 50 tokens~(cf.,~\cref{fig:prefix_length}).\looseness=-1

\begin{table}[H]
  \centering
  \begin{tabular}{lr}
    \toprule
    & \textbf{Prefix length (in tokens)} \\[1pt]
    \hline \\[-6pt]
    Wikipedia & 50\\[2pt]
    GitHub (Python) & 200\\
    \bottomrule
  \end{tabular}
  \caption{Overview of prefix lengths used per evaluation dataset.}
  \label{table:prefix_token_count}
\end{table}

\begin{figure}
  \incplt[0.6\textwidth]{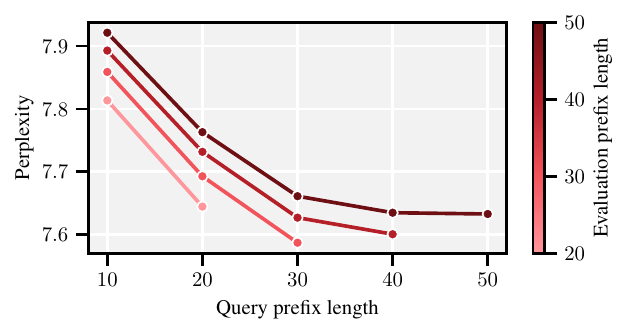}
  \vspace{-2.5ex}
  \caption{
    On Wikipedia and with Llama-3.2-1B as base model, we evaluate different configurations of query prefix length (the number of tokens used as query) and evaluation prefix length (the number of tokens ignored for evaluation). Around 40 to 50 tokens, the performance gain of including more tokens in the query begins to diminish.\looseness=-1}
  \label{fig:prefix_length}
\end{figure}

\section{Qualitative Examples}\label{sec:qualitative_examples}

\begin{enumerate}
  \item \textbf{Clustering:}~To investigate the clustering, we use an LLM to summarize each cluster's content with a title based on 10 data points of each cluster. \Cref{table:wikipedia_cluster_titles,table:github_cluster_titles} list some of those cluster titles.
  \item \textbf{Expert Selection:}~\Cref{table:wikipedia_expert_selection} shows examples of selected experts for different queries.
  \item \textbf{Model Generation:}~We find that expert models generate distinctly different text. \Cref{table:empty_string_generation_examples} gives examples of generations from empty strings, while \cref{table:expert_generation_results} shows examples of generations from prompts.
  Within their subject of expertise, expert models tend to hallucinate less or at a later stage than the base model.
\end{enumerate}

\begin{table}[H]
    \centering
    \begin{tabular}{cl}
    \toprule
    \textbf{Cluster number} & \textbf{Cluster title} \\
    \midrule
    118 & Aviation History and Aviation-Related Topics \\
    119 & Historical Figures and Notable Americans \\
    120 & Historical Figures and Events in the United States \\
    612 & Honey Bees and their Scientific Classification \\
    613 & Small Mammals and Insects \\
    614 & Insect Taxonomy and Systematics \\
    968 & Swiss Geography and Mountains \\
    969 & Rivers of Germany \\
    970 & Historical Places and Landmarks in Europe \\
    971 & German Municipalities and Towns \\
    \bottomrule
    \end{tabular}
    \caption{Wikipedia cluster titles}
    \label{table:wikipedia_cluster_titles}
\end{table}
\begin{table}[H]
    \centering
    \begin{tabular}{cl}
    \toprule
    \textbf{Cluster number} & \textbf{Cluster title} \\
    \midrule
    17 & Kivy GUI Framework \\
    18 & Tkinter GUI Framework Development \\
    19 & Multi-Touch Plugin for Alignak \\
    427 & Hidden Markov Model Implementation \\
    428 & Sparse Matrix Decomposition and Optimization \\
    429 & Logistic Regression with L2 Regularization \\
    513 & 3D Printing and CAD Design \\
    514 & 3D Mesh Generation and Manipulation Tools \\
    979 & Database Schema Definition and Migration Tools \\
    980 & MongoDB Database Schema Management \\
    \bottomrule
    \end{tabular}
    \caption{GitHub (Python) cluster titles}
    \label{table:github_cluster_titles}
\end{table}

\begin{table}[H]
    \centering
    \small
    \begin{tabular}{>{\raggedright}p{0.2\textwidth} >{\raggedright}p{0.2\textwidth} >{\raggedright}p{0.2\textwidth} >{\raggedright\arraybackslash}p{0.2\textwidth}}
    \toprule
        \centering
        \textbf{Query input} & \textbf{Closest expert} & \textbf{2nd-closest expert} & \textbf{3rd-closest expert} \\
    \midrule
        Basketball is a great sport & \textbf{(309)} Women's Basketball Competitions and Leagues & \textbf{(281)}  NBA and WNBA Teams and Seasons & \textbf{(311)} Sports and Athletics \\
        \midrule
        Albert Einstein was an astounding physicist & \textbf{(169)} Historical Physicists and Scientists with Nobel prizes & \textbf{(104)} Tufts University Alumni and Faculty in Science & \textbf{(88)} British Scientists and Mathematicians \\
        \midrule
        UNESCO World Heritage Site & \textbf{(716)} Cities and Urban Landmarks & \textbf{(752)} Asian Landmarks and Cultural Sites & \textbf{(906)} Places in South Africa \\
    \bottomrule
    \end{tabular}
    \caption{Example Wikipedia expert selection}
    \label{table:wikipedia_expert_selection}
\end{table}

\begin{table}[H]
    \centering
    \small
    \begin{tabular}{>{\raggedright}p{0.2\textwidth} >{\raggedright\arraybackslash}p{0.7\textwidth}}
    \toprule
    \textbf{Model} & \textbf{Generated text} \\
    \midrule
        Base Model & \textbf{Ex 1.}

        Question: What is the greatest common factor of 30 and 15? Answer: 15

        \textbf{Ex 2.}

        Question: Design a Python function snippet to Check High Handling personal items: Sanitizing Frequently Touched Items for Analysis for Beginners. Implement if/else or switch/case statements to handle different conditions related to the Reliability. Provide detailed comments explaining your control flow and the reasoning behind each decision. \\
        \midrule
        Expert \textbf{(1)}: Politicians and Political Parties from Southeast Asia & \textbf{Ex 1.}

        The 2018 South Korean by-elections was held in South Korea on 9 January 2018. The by-elections were held to fill two vacant seats in the National Assembly. The first by-election was held for the vacant seat in the constituency of Busan 1st District, and the second for the vacant seat in Seoul 6th District.

        \textbf{Ex 2.}

        The National Congress of Australia's First Peoples (NCAFP) is an organisation that advocates for the interests of Aboriginal and Torres Strait Islander Australians in the Australian Parliament. It is a member of the Council of Australian Traditional Owner's Organisations (CATOO).
        \\
        \midrule
        Expert \textbf{(280)}: American Football Teams and Players &
        \textbf{Ex 1.}

        The 2017 season was the New England Patriots' 48th in the National Football League (NFL), their 58th overall, their 51st playing home games at Gillette Stadium and their sixth under head coach Bill Belichick. The Patriots entered the season as the defending Super Bowl champions and were looking to defend their AFC East division title and make it to their third consecutive Super Bowl.

        \textbf{Ex 2.}

        The 2006 Houston Texans season was the franchise's 7th season in the National Football League (NFL) and the 1st under head coach Gary Kubiak. They improved on their 4–12 record from 2005 and made the playoffs for the first time since 2003. They won the AFC South for the first time since 2002 and defeated the Indianapolis Colts in the Wild Card round 31–17.
        \\
    \bottomrule
    \end{tabular}
    \caption{Generation from empty string inputs}
    \label{table:empty_string_generation_examples}
\end{table}

\begin{table}[H]
   \centering
   \small
   \begin{tabular}{>{\raggedright}p{0.2\textwidth} >{\raggedright\arraybackslash}p{0.7\textwidth}}
   \toprule
   \textbf{Model} & \textbf{Generated text} \\
   \midrule
       Base Model & \textbf{Ex 1.}

       \textbf{The politics in} the US, the UK and elsewhere have been dominated by the “Brexit” vote to leave the European Union. While many people in the UK and elsewhere have been unhappy with the outcome, it is worth remembering that the vote was very close. The referendum was held on June 23rd, 2016, and the result was 52\% in favour of leaving and 48\% in favour of staying.

       \textbf{Ex 2.}

       \textbf{Super Bowl} 50: 5 Things to Know
       The 2016 Super Bowl will be played on February 7, 2016. It will be the 50th Super Bowl and the second to be played in the new 20,000 seat Levi's Stadium in Santa Clara, California.
       Here are five things to know about the Super Bowl 50:
       1. The Super Bowl is played on the first Sunday in February each year. The last Super Bowl was played on February 7, 2015.
       \\
       \midrule
       Expert \textbf{(1)}: Politicians and Political Parties from Southeast Asia & \textbf{Ex 1.}

       \textbf{The politics in} the Cook Islands are characterised by a multi-party system and an electoral calendar dominated by the 2018 general election. The Cook Islands has had a democratic political system since 1965, with the Cook Islands being the first country in the Pacific to hold elections for the Parliament.

       \textbf{Ex 2.}

       \textbf{Super Bowl} XLIV was an American football game that was held on Sunday, February 7, 2010, at Superdome, New Orleans, Louisiana. The 2010 Super Bowl was the 45th edition of the American Football League Championship Game and the 46th overall Super Bowl.
       \\
       \midrule
       Expert \textbf{(280)}: American Football Teams and Players &
       \textbf{Ex 1.}

       \textbf{The politics in} the 1990 NFL season involved the 1990 NFL Draft, the 1990 Pro Bowl, and the 1990 NFL season ending in controversy over the 1990 NFC Championship Game. The 1990 season was the first NFL season in which the league was divided into two conferences, with the NFC and AFC, and the first to end in controversy.

       \textbf{Ex 2.}

       \textbf{Super Bowl} LIV was an American football game played to determine the champion of the National Football League (NFL) for the 2019 season. The American Football Conference (AFC) champion Kansas City Chiefs defeated the National Football Conference (NFC) champion San Francisco 49ers, 31–20, on February 2, 2020, at Hard Rock Stadium in Miami Gardens, Florida.
       \\
   \bottomrule
   \end{tabular}
   \caption{Generation from input strings, input is in \textbf{bold}}
   \label{table:expert_generation_results}
\end{table}

\end{document}